%% file: main.tex
\documentclass[letterpaper]{article} %
\usepackage[preprint]{aaai2027}  %
\usepackage[hyphens]{url}  %
\usepackage{graphicx} %
\usepackage{natbib}  %
\usepackage{caption} %
\usepackage{amsmath}                      %
\usepackage{amssymb}                      %
\usepackage{amsthm}                       %
\newtheorem{proposition}{Proposition}[section]
\newtheorem{lemma}{Lemma}[section]
\usepackage{booktabs}                     %
\usepackage{multirow}                     %
\usepackage{xcolor}                       %
\graphicspath{{figures/}}                 %

\newcommand{\methodname}{PeekKV}          %

\title{PeekKV: The First Content Token Already Knows Which Heads Matter}
\author{Minyoung Go, Yongjun Park}
\affiliations{
    Yonsei University\\   %
    minyoung.go@yonsei.ac.kr \\
    yongjunpark@yonsei.ac.kr
}

\begin{document}

\maketitle

\begin{abstract}
\input{sections/00-abstract}
\end{abstract}

\input{sections/01-introduction}
\input{sections/02-background}
\input{sections/03-method}
\input{sections/04-experiments}
\input{sections/06-conclusion}

\bibliography{reference}   %

\appendix
\setcounter{equation}{0}
\renewcommand{\theequation}{S\arabic{equation}}
\input{supplementary/sections/A-proofs}
\input{supplementary/sections/B-implementation}
\input{supplementary/sections/C-reproducibility}
\input{supplementary/sections/D-results}

\end{document}

%% file: sections/00-abstract.tex
The size of the KV cache grows linearly with context length, making it a major memory bottleneck in LLM serving. Motivated by recent progress in identifying the functional roles of attention heads, a growing body of work compresses the KV cache non-uniformly across heads. Existing methods either fix the important heads via offline profiling or select them using attention signals from the prefill stage. In both cases, however, accuracy fluctuates substantially across workloads. While investigating the cause of this instability, we find that the active heads shift between the prefill and decode stages even within a single input. Along the decode path, by contrast, the active-head composition remains stable across content tokens once format tokens are filtered out. Based on this observation, we propose \methodname{}. \methodname{} observes a single content token immediately after prefill — a one-token peek — to jointly obtain head-group scores and token-importance scores. Guided by these scores, \methodname{} allocates KV cache budgets to head groups in proportion to their group scores and retains only the top-scoring tokens within each budget. \methodname{} requires no offline profiling or training and thus applies directly to unseen workloads and new models. Across four public benchmarks, \methodname{} achieves the highest accuracy in more than half of the per-model comparisons against state-of-the-art methods, and stays closest to the full-cache accuracy on RULER and LongBench. On Llama-3.1-8B-Instruct, \methodname{} retains 95.7\% of the full-cache accuracy at a 10\% memory budget on RULER with 16K contexts and 95.9\% at a 1.5\% budget on LongBench.

%% file: sections/01-introduction.tex
\section{Introduction}
\label{sec:intro}

As LLMs move beyond simple chat to serve as agents, input contexts have grown longer, and it has been necessary to maintain states across long-running tasks. For this, the KV cache scales with context length, straining both the memory capacity and bandwidth of serving hardware.

Efforts to reduce the KV cache without sacrificing accuracy have proceeded along several axes. Along the token axis, H2O~\citep{h2o}, TOVA~\citep{tova}, and SnapKV~\citep{snapkv} evict tokens based on attention scores. PyramidKV~\citep{pyramidkv} and ThinK~\citep{think} extend compression to the layer and dimension axes, respectively. Rather than evicting tokens, KIVI~\citep{kivi} and TurboQuant~\citep{turboquant} compress the cache via quantization.

In parallel, mechanistic evidence has accumulated on the functional roles of individual attention heads, including induction heads~\citep{inductionheads}, retrieval heads~\citep{retrievalhead}, and function vector heads~\citep{functionvector}. Building on these findings, DuoAttention~\citep{duoattention}, RazorAttention~\citep{razorattention}, and HeadKV~\citep{headkv} allocate large token budgets to important induction and retrieval heads while aggressively pruning the rest.

However, these head-level eviction approaches face several limitations. First, they require costly per-model offline profiling to determine which heads to preserve. Second, which heads matter varies across workloads, so statically fixing head importance and budgets shows suboptimal results. To address this, AdaKV~\citep{adakv} adaptively allocates KV budgets across heads using attention distributions from the prefill stage. In our measurements, however, this adaptive approach identifies important heads even less accurately than its static counterpart~\citep{headkv} on ten of the twelve workload categories we evaluate (Figure~\ref{fig:headworkload}a). Indeed, we observe that the active heads shift between the prefill and decode stages even within a single input (Section~\ref{sec:background}). So long as the two stages activate different heads, any selection based on prefill signals is misaligned with the heads actually needed at decode.

We therefore propose \methodname{}, which reads this signal directly from the decode path. \methodname{} lets the model briefly continue generation immediately after prefill, skips format tokens, and observes per-head attention signals only at the first content token encountered (one-token peek). A single peek thus yields two kinds of scores at once. \methodname{} allocates KV budgets to head groups in proportion to their group scores, and the output-perturbation ($\Delta o$) scores determine tokens retained by each group.

The resulting allocation adapts to each workload differently. On recall-dominant tasks, induction and retrieval heads score high and receive large KV budgets. On other tasks such as tool calling or code generation, budgets concentrate on the heads each task requires, and the rest are compressed to a minimal budget. Since which heads receive budgets is determined at runtime for each input, \methodname{} can be applied directly to new workloads without offline profiling or training. This score-proportional budget allocation concentrates capacity on a few important heads, keeping \methodname{} competitive even on long-context retrieval, where pure eviction methods are typically not sufficienlty effective.

We evaluate \methodname{} on four public benchmarks spanning long-context retrieval as well as text and code generation. On Llama-3.1-8B-Instruct, \methodname{} retains 95.7\% of the full-cache accuracy at a 10\% memory budget on RULER, and 95.9\% at a 1.5\% budget on LongBench. The overhead of the one-token peek is limited to deferring compression by a few decoding steps, during which peak memory remains at the full-cache level.

Our contributions are as follows.
\begin{itemize}
  \item \textbf{(C1)} Going beyond the observation that head activity fluctuates even within a single input~\citep{dynretrieval}, we present the first evidence, to the best of our knowledge, that this fluctuation concentrates at the prefill--decode boundary and on format tokens, whereas the active heads remain stable across content tokens.
  \item \textbf{(C2)} We propose a method that obtains both head-group and output-perturbation ($\Delta o$) scores from a single one-token peek, without offline profiling or training. The group scores set each group's KV budget, and the $\Delta o$ scores decide which tokens it retains.
  \item \textbf{(C3)} Across four public benchmarks, \methodname{} achieves the highest accuracy in more than half of the per-model comparisons and stays closest to the full-cache accuracy on RULER and LongBench.
\end{itemize}

%% file: sections/02-background.tex
\section{Background and Key Observations}
\label{sec:background}

\begin{figure}[t]
  \centering
  \includegraphics[width=\linewidth]{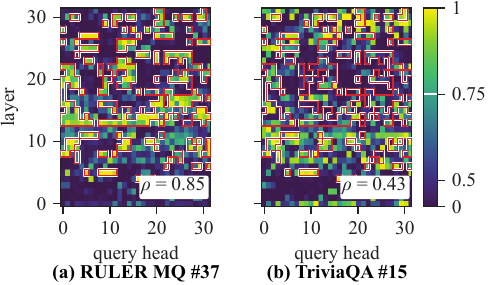}
  \caption{Example decode-time head signal maps. Red outlines mark the top-25\% heads statically selected by HeadKV, and $\rho$ is the rank correlation with the HeadKV map.}
  \label{fig:headheatmap}
\end{figure}

\begin{figure}[t]
  \centering
  \includegraphics[width=\linewidth]{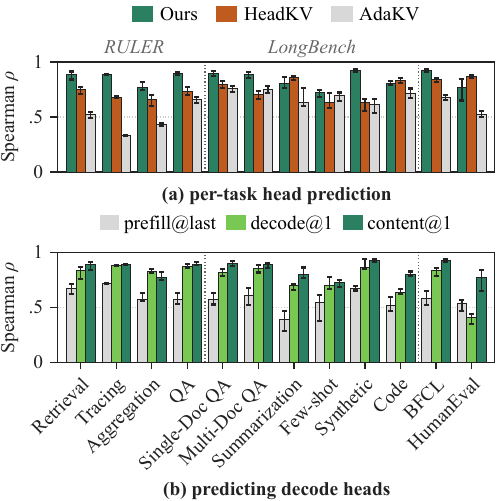}
  \caption{Per-item agreement with the full-decode head ranking.}
  \label{fig:headworkload}
\end{figure}

\paragraph{Head selection.}
\citet{inductionheads} identifies induction heads, which detect repeated patterns in context and copy the next token, and presents correlational evidence that they emerge naturally during training and are linked to in-context learning. Extending this analysis, subsequent work characterizes additional head types, including retrieval heads~\citep{retrievalhead} and function vector heads~\citep{functionvector}.

Among these, retrieval heads are reported to drive search and recall within the context~\citep{retrievalhead}. Exploiting this property, DuoAttention~\citep{duoattention} and RazorAttention~\citep{razorattention} show that retaining the full context for retrieval heads largely preserves recall even when the remaining heads are compressed or pruned. HeadKV~\citep{headkv} allocates KV budgets non-uniformly across heads according to offline-profiled retrieval scores. Meanwhile, because important heads vary with the workload, other studies argue for identifying them dynamically. AdaKV~\citep{adakv} instantiates this idea as a plug-in that adapts per-head budgets to prefill-time attention distributions. \citet{dynretrieval} further reports that retrieval head activity fluctuates across decoding timesteps and is therefore difficult to capture with static selection.

Our own measurements reproduce this variability in head importance at the level of individual inputs. As shown in Figure~\ref{fig:headheatmap}(a), retrieval heads play the leading role on recall-dominant tasks, which explains why heavily compressing all but these heads leaves the output largely unaffected~\citep{duoattention,razorattention}. By contrast, on other workloads non-retrieval heads also play leading roles in the output (Figure~\ref{fig:headheatmap}(b)), and the same preservation strategy can miss the heads that matter. Together, these observations suggest that fixing which heads to preserve offline often shows suboptimal results for such workloads.

\paragraph{Prefill--Decode Head Shift.}
We go one step further: within a single input, the heads playing the leading roles shift between the prefill and decode stages. To the best of our knowledge, we are the first to localize the head activity fluctuation to this stage boundary. To quantify this, we score each head's activity by how strongly it attends to its single most-attended context position (excluding the attention sink~\citep{streamingllm,attentionsink} at the front and the always-kept recent window at the tail), and rank the heads by this score accumulated over the entire decode stage to obtain the reference ranking. We then compare how well different early proxy signals predict this reference ranking. As shown in Figure~\ref{fig:headworkload}(b), the signal at the last prefill token (prefill@last) attains an average rank correlation of only 0.58 with the reference ranking, whereas the signal at the first decode token (decode@1) — just one step later — reaches 0.77.

\input{figures/fig_token_regions/fig_token_regions.tex}

\begin{table}[t]
  \centering
  \small
  \begin{tabular}{@{}lll@{}}
    \toprule
    Domain & Form & Examples \\
    \midrule
    Code & Syntax elements & \texttt{import}, \texttt{def} \\
    Documents & Markup syntax & \texttt{\textasciigrave\textasciigrave\textasciigrave\textless lang\textgreater}, \texttt{\textasciigrave\textasciigrave\textasciigrave} \\
    Ordinary text & Articles and punctuation & \textit{a}, \textit{an}, \textit{the}, \texttt{:}, \texttt{;} \\
    \bottomrule
  \end{tabular}
  \caption{Representative format token types and examples.}
  \label{tab:formattokens}
\end{table}

\paragraph{Format and Content Tokens.}
On code generation benchmarks such as HumanEval, however, decode@1 predicts the reference ranking no better than prefill@last — and sometimes even worse (Figure~\ref{fig:headworkload}(b)). Investigating the cause, we find that the attention signals at format-defining tokens are largely unrelated to the heads that matter for the overall generation. We therefore distinguish two kinds of decode tokens, \emph{format tokens} that carry only formatting information and \emph{content tokens} that carry actual content (Figure~\ref{fig:tokenregions}, Table~\ref{tab:formattokens}). The distinction is made by syntax rather than by a fixed word list. In code, a tree-sitter parser~\citep{treesitter} analyzes the partially generated text and marks its keywords and delimiters as format, leaving identifiers and literals as content. Opening code fences, articles, and punctuation are classified as format as well. On LongBench and RULER, where format tokens amount to little more than articles, decode@1 and content@1 — the signal at the first content token — predict the reference ranking almost equally well. On HumanEval, by contrast, an average of approximately 5.06 format tokens precede the first content token. The correlation of raw decode@1 drops sharply to 0.41, then recovers to 0.77 at content@1. This suggests that the head selection signal should be extracted from the first content token, once format tokens are filtered out. We build \methodname{} directly on these observations in the next section.

%% file: figures/fig_token_regions/fig_token_regions.tex
\begin{figure}[t]
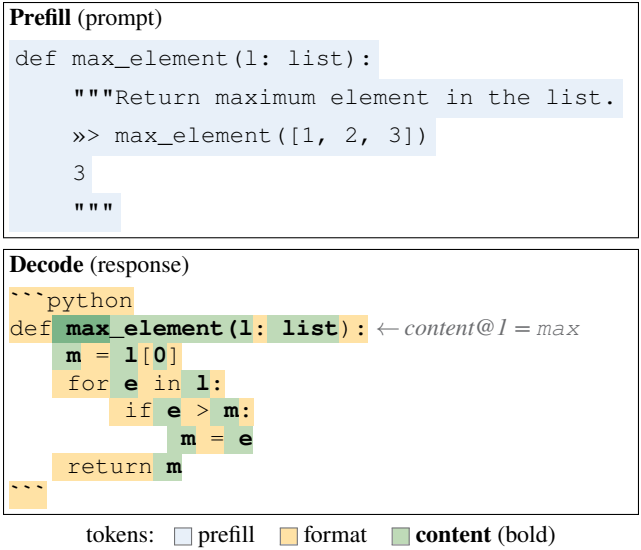

  \centering
  \definecolor{hlprefill}{cmyk}{0.08,0.03,0,0}%
  \definecolor{hlformat}{cmyk}{0,0.10,0.40,0}%
  \definecolor{hlcontent}{cmyk}{0.22,0,0.30,0.05}%
  \definecolor{hlcontenttok}{cmyk}{0.45,0,0.50,0.15}%
  \definecolor{noteink}{cmyk}{0,0,0,0.70}%
  \setlength{\fboxsep}{2pt}%
  \fbox{\begin{minipage}[t]{\dimexpr\columnwidth-2\fboxsep-2\fboxrule\relax}
    \raggedright\setlength{\parindent}{0pt}%
    {\small\textbf{Prefill} (prompt)}\par\smallskip
    \fontsize{9}{10.4}\selectfont
    \begin{tabular}{@{}l@{}}
\colorbox{hlprefill}{\strut\ttfamily def\phantom{0}max\_element(l:\phantom{0}list):}\\
\colorbox{hlprefill}{\strut\ttfamily \phantom{0000}\symbol{34}\symbol{34}\symbol{34}Return\phantom{0}maximum\phantom{0}element\phantom{0}in\phantom{0}the\phantom{0}list.}\\
\colorbox{hlprefill}{\strut\ttfamily \phantom{0000}>>>\phantom{0}max\_element([1,\phantom{0}2,\phantom{0}3])}\\
\colorbox{hlprefill}{\strut\ttfamily \phantom{0000}3}\\
\colorbox{hlprefill}{\strut\ttfamily \phantom{0000}\symbol{34}\symbol{34}\symbol{34}}\\
    \end{tabular}
  \end{minipage}}\par\smallskip
  \fbox{\begin{minipage}[t]{\dimexpr\columnwidth-2\fboxsep-2\fboxrule\relax}
    \raggedright\setlength{\parindent}{0pt}%
    {\small\textbf{Decode} (response)}\par\smallskip
    \fontsize{9}{10.4}\selectfont
    \begin{tabular}{@{}l@{}}
\ttfamily {\setlength{\fboxsep}{0pt}\colorbox{hlformat}{\strut \textasciigrave{}\textasciigrave{}\textasciigrave{}}}{\setlength{\fboxsep}{0pt}\colorbox{hlformat}{\strut python}}\\
\ttfamily {\setlength{\fboxsep}{0pt}\colorbox{hlformat}{\strut def}}{\setlength{\fboxsep}{0pt}\colorbox{hlcontenttok}{\strut \bfseries \phantom{0}max}}{\setlength{\fboxsep}{0pt}\colorbox{hlcontent}{\strut \bfseries \_element}}{\setlength{\fboxsep}{0pt}\colorbox{hlcontent}{\strut \bfseries (l}}{\setlength{\fboxsep}{0pt}\colorbox{hlformat}{\strut :}}{\setlength{\fboxsep}{0pt}\colorbox{hlcontent}{\strut \bfseries \phantom{0}list}}{\setlength{\fboxsep}{0pt}\colorbox{hlformat}{\strut ):}}\hspace{3pt}{\normalfont\itshape\color{noteink}$\leftarrow$\,content@1 = \texttt{max}}\\
\ttfamily \phantom{000}{\setlength{\fboxsep}{0pt}\colorbox{hlcontent}{\strut \bfseries \phantom{0}m}}{\setlength{\fboxsep}{0pt}\colorbox{hlformat}{\strut \phantom{0}=}}{\setlength{\fboxsep}{0pt}\colorbox{hlcontent}{\strut \bfseries \phantom{0}l}}{\setlength{\fboxsep}{0pt}\colorbox{hlformat}{\strut [}}{\setlength{\fboxsep}{0pt}\colorbox{hlcontent}{\strut \bfseries 0}}{\setlength{\fboxsep}{0pt}\colorbox{hlformat}{\strut ]}}\\
\ttfamily \phantom{000}{\setlength{\fboxsep}{0pt}\colorbox{hlformat}{\strut \phantom{0}for}}{\setlength{\fboxsep}{0pt}\colorbox{hlcontent}{\strut \bfseries \phantom{0}e}}{\setlength{\fboxsep}{0pt}\colorbox{hlformat}{\strut \phantom{0}in}}{\setlength{\fboxsep}{0pt}\colorbox{hlcontent}{\strut \bfseries \phantom{0}l}}{\setlength{\fboxsep}{0pt}\colorbox{hlformat}{\strut :}}\\
\ttfamily \phantom{0000000}{\setlength{\fboxsep}{0pt}\colorbox{hlformat}{\strut \phantom{0}if}}{\setlength{\fboxsep}{0pt}\colorbox{hlcontent}{\strut \bfseries \phantom{0}e}}{\setlength{\fboxsep}{0pt}\colorbox{hlformat}{\strut \phantom{0}>}}{\setlength{\fboxsep}{0pt}\colorbox{hlcontent}{\strut \bfseries \phantom{0}m}}{\setlength{\fboxsep}{0pt}\colorbox{hlformat}{\strut :}}\\
\ttfamily \phantom{00000000000}{\setlength{\fboxsep}{0pt}\colorbox{hlcontent}{\strut \bfseries \phantom{0}m}}{\setlength{\fboxsep}{0pt}\colorbox{hlformat}{\strut \phantom{0}=}}{\setlength{\fboxsep}{0pt}\colorbox{hlcontent}{\strut \bfseries \phantom{0}e}}\\
\ttfamily \phantom{000}{\setlength{\fboxsep}{0pt}\colorbox{hlformat}{\strut \phantom{0}return}}{\setlength{\fboxsep}{0pt}\colorbox{hlcontent}{\strut \bfseries \phantom{0}m}}\\
\ttfamily {\setlength{\fboxsep}{0pt}\colorbox{hlformat}{\strut \textasciigrave{}\textasciigrave{}\textasciigrave{}}}\\
    \end{tabular}
  \end{minipage}}\par\smallskip
  \makebox[\columnwidth][c]{\small
    \setlength{\fboxsep}{1pt}%
    tokens:\hspace{5pt}
    \fcolorbox{noteink}{hlprefill}{\phantom{\rule{4pt}{4pt}}}~prefill\hspace{7pt}
    \fcolorbox{noteink}{hlformat}{\phantom{\rule{4pt}{4pt}}}~format\hspace{7pt}
    \fcolorbox{noteink}{hlcontent}{\phantom{\rule{4pt}{4pt}}}~\textbf{content} (bold)}
  \caption{Example prefill, format, and content token regions (HumanEval/35). content@1 marks the first content token.}
  \label{fig:tokenregions}
\end{figure}

%% file: sections/03-method.tex
\section{Method}
\label{sec:method}

\subsection{Preliminaries}
\label{sec:method-prelim}

For notational simplicity, we fix a single layer and omit the layer index, layer normalization, positional encoding (RoPE), and bias terms. Let $x_t \in \mathbb{R}^{d}$ denote the hidden state input to the attention block at position $t$, where $d$ is the model dimension. The block has $H$ query heads. Under grouped-query attention (GQA)~\citep{gqa}, adopted by most open-weight models today, such as Llama~3, Mistral, and Qwen, the query heads are partitioned into $H_{\mathrm{kv}}$ groups, each sharing one set of key/value projections. Let $g(h)$ be the group containing head $h$. The query, key, and value vectors are then computed with the projections $W_Q^{(h)}, W_K^{(g)}, W_V^{(g)} \in \mathbb{R}^{d_h \times d}$, where $d_h$ is the head dimension and $g \in \{1,\dots,H_{\mathrm{kv}}\}$ indexes the groups:
\begin{equation}
\label{eq:qkv}
q_t^{(h)} = W_Q^{(h)} x_t, \qquad k_t^{(g)} = W_K^{(g)} x_t, \qquad v_t^{(g)} = W_V^{(g)} x_t.
\end{equation}
The KV cache thus holds a single $(k_j, v_j)$ pair per group, shared by all of its query heads. An entry's memory is freed only when every one of these heads discards it, so the KV to retain must be selected at the group level.

At decode step $t$, head $h$ computes scaled dot products between its query and the keys at positions $j \le t$, as permitted by the causal mask, and normalizes them with a softmax to obtain the attention distribution:
\begin{equation}
\label{eq:attnweight}
\alpha_{t,j}^{(h)} = \frac{\exp\big(q_t^{(h)\top} k_j^{(g(h))} / \sqrt{d_h}\big)}{\sum_{j' \le t} \exp\big(q_t^{(h)\top} k_{j'}^{(g(h))} / \sqrt{d_h}\big)}.
\end{equation}
The attention output of head $h$ is the sum of the value vectors weighted by this distribution:
\begin{equation}
\label{eq:attnout}
o_t^{(h)} = \sum_{j \le t} \alpha_{t,j}^{(h)} \, v_j^{(g(h))}.
\end{equation}
The pairs $(k_j, v_j)$ for $j < t$ are read from the KV cache rather than recomputed at every step. Both the memory the cache occupies and the bandwidth spent reading it at each step therefore grow in proportion to the number of tokens it retains. This is the bottleneck that KV cache compression targets.

The per-head outputs are concatenated, passed through the output projection $W_O \in \mathbb{R}^{d \times H d_h}$, and added to the residual stream~\citep{mathframework}. Splitting $W_O$ into per-head blocks $W_O^{(h)} \in \mathbb{R}^{d \times d_h}$ decomposes the attention block's output exactly into a sum of per-head contributions:
\begin{equation}
\label{eq:headdecomp}
\mathrm{Attn}(x_t) = W_O \big[ o_t^{(1)}; \cdots; o_t^{(H)} \big] = \sum_{h=1}^{H} W_O^{(h)} o_t^{(h)}.
\end{equation}
In this decomposition, the direct contribution of head $h$ to the residual stream is $W_O^{(h)} o_t^{(h)}$.

\subsection{One-Token Peek: Head-Group Scoring}
\label{sec:method-peek}

Following the observations in Section~\ref{sec:background}, \methodname{} reads the head selection signal from a content token on the decode path. Immediately after prefill, the model briefly continues generation on its own, with the KV cache left uncompressed at this stage. \methodname{} skips the format tokens and reads the signal at the first content token it encounters. Format tokens are identified by predefined rules.\footnote{The full list is given in the supplementary material} We call this procedure the \emph{one-token peek}. The peek ends once it observes the first content token.

At the observation step $t^\ast$ (the first content token), the score of head $h$ is defined by how strongly it attends to a single context position:
\begin{equation}
\label{eq:headscore}
s^{(h)} = \max_{j \in \mathcal{C}} \alpha_{t^\ast,j}^{(h)},
\end{equation}
where $\mathcal{C}$ is the set of prefill context positions excluding the attention sink span~\citep{streamingllm,attentionsink} at the front and the always-kept recent window at the tail, with the lengths of both spans left as hyperparameters. This statistic is the same one used to score head activity in Section~\ref{sec:background}. Based on the design, $s^{(h)}$ is high for heads that focus on a specific context position, and low for heads whose attention is diffuse or concentrated on the sink. 
Assigning a low score to diffuse heads incurs little accuracy loss.
Diffuse attention spreads small weights over many positions, and even low-scoring groups retain the recent window and a few top-scoring tokens, so the output is nearly unchanged.

Since the KV cache is stored in the group level (Section~\ref{sec:method-prelim}), selection is also performed at the group level. The score $S^{(g)}$ of group $g$ is the sum of the scores $s^{(h)}$ of its query heads. To compare groups across layers, we now restore the layer index omitted in Section~\ref{sec:method-prelim}. This gives $L\,H_{\mathrm{kv}}$ groups in total, where $L$ is the number of layers. Their scores are normalized into a single distribution, which determines the score-proportional budget allocation in Section~\ref{sec:method-configs}. The tokens generated during the peek, along with their KV, are retained rather than discarded. After the prefill prefix is compressed, the main generation resumes from where the peek stopped.

\subsection{Output-Perturbation Token Scoring}
\label{sec:method-tokscore}

While the group scores govern the budget allocation across groups, deciding which tokens to retain within each group requires a token-level signal. A widely used approach selects tokens by their attention scores in a prefill observation window~\citep{snapkv}, but prefill queries differ from decode queries, so the selected tokens can be misaligned with those the subsequent generation actually needs. This is the same prefill--decode misalignment seen for head selection (Section~\ref{sec:background}). \methodname{} therefore obtains token scores from the same peek observation used for the head scores.

We define the importance of a token by how much the attention output changes when that token is removed, a quantity we call the output perturbation ($\Delta o$). For notational simplicity, we omit the observation step subscript $t^\ast$ in this section. Removing a single token $j$ from the attention of head $h$ renormalizes the remaining weights, so the resulting change in the attention output at that step can be computed exactly in closed form:
\begin{equation}
\label{eq:deltao}
\Delta o_j^{(h)} = \frac{\alpha_j^{(h)}}{1 - \alpha_j^{(h)}} \big( o^{(h)} - v_j^{(g(h))} \big),
\end{equation}
where $\alpha_j^{(h)}$ is the attention weight from Eq.~\eqref{eq:attnweight} and $o^{(h)}$ is the attention output from Eq.~\eqref{eq:attnout}. The effect of this output perturbation on the residual stream is measured by passing it through the per-head projection $W_O^{(h)}$ from the decomposition~\eqref{eq:headdecomp}:
\begin{equation}
\label{eq:tokscore}
s_j^{(h)} = \lVert W_O^{(h)} \Delta o_j^{(h)} \rVert_2 \!=\! \frac{\alpha_j^{(h)}}{1 - \alpha_j^{(h)}} \lVert W_O^{(h)} ( o^{(h)} - v_j^{(g(h))} ) \rVert_2.
\end{equation}

The squared norm is expanded with the Gram matrix $G^{(h)} = W_O^{(h)\top} W_O^{(h)}$, which avoids an explicit $W_O^{(h)}$ projection per token and lowers the per-token cost by a factor of $d/d_h$.\footnote{The supplementary material derives Eqs.~\eqref{eq:deltao} and~\eqref{eq:tokscore} and analyzes the cost of this expansion.} This score builds on the same algebra as the token-perturbation identity of CAOTE~\citep{caote}, stated in value space. CriticalKV~\citep{criticalkv} formulates the selection problem not as independent per-token scoring but as minimizing the worst-case output perturbation of the entire selected set. Our approach differs in that we measure the perturbation once, at the first content token on the decode path.

This scoring formula is common to all groups and is applied to each group independently. As with the head scores, selection is performed at the group level. The token score $s_j^{(g)}$ of group $g$ aggregates the scores $s_j^{(h)}$ of its query heads by summation. Each group retains the KV of its top-scoring tokens, with the number set by its budget. The candidates are all prefill tokens except the always-kept recent window at the tail. For the head scores (Eq.~\eqref{eq:headscore}), the attention sink was excluded from the candidate set $\mathcal{C}$ to avoid distorting the scores. Here, however, the goal is to choose what to preserve, so the sink is included among the candidates. Since the sink dominates the attention mass $\alpha_j^{(h)}$, it tends to rank high and be retained on its own, so the implementation has no separate sink protection rule. In our setup, where the peek observes only the first content token, we use the scores from step $t^\ast$ directly. Compression is applied once to the prefill prefix, and the KV of tokens newly generated during the main generation is kept intact.

\begin{figure*}[t]
  \centering
  \includegraphics[width=\textwidth]{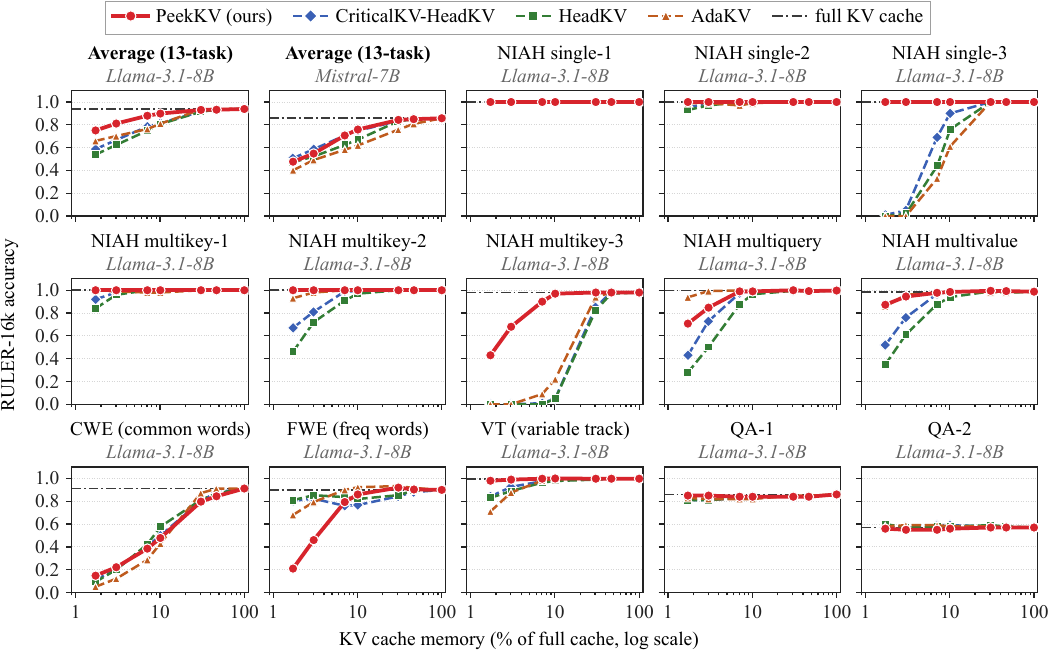}
  \caption{Per-task budget--accuracy curves on RULER-16k, with the budget as a fraction of the full cache.}
  \label{fig:rulerbudgetgrid}
\end{figure*}
\subsection{Deployment: Score-Proportional Budget Allocation}
\label{sec:method-configs}

KV cache budgets are allocated to groups in proportion to their group scores from the peek, and the $\Delta o$ token scores determine which tokens each group retains. Let $\hat{S}^{(g)} = S^{(g)} / \sum_{g'=1}^{L H_{\mathrm{kv}}} S^{(g')}$ denote the group score $S^{(g)}$ (Section~\ref{sec:method-peek}) normalized over all groups. In proportion to this score, each group $g$ receives the budget
\begin{equation}
\label{eq:scorecap}
c_g = \min\big(r + \hat{S}^{(g)}\, P,\; c\big).
\end{equation}
In this equation, $r$ is the floor that every group receives at minimum, and $c$ is the ceiling that any single group can receive. The slope is set to $P = L\,H_{\mathrm{kv}}\,(b - r)$ so that the average per-group budget before applying the ceiling equals the target $b$, which is the number of retained tokens per group (using $\sum_g \hat{S}^{(g)}=1$). Since $\hat{S}^{(g)} \ge 0$, no separate lower-bound clipping is needed. The prefill tokens of each group fall into three classes. \emph{Recent-window tokens} occupy a fixed span at the tail of the prefill and are always kept outside the budget $c_g$. \emph{Pinned tokens} fill each group's budget $c_g$ in decreasing order of their $\Delta o$ scores (Eq.~\eqref{eq:tokscore}). \emph{Evicted tokens} are those exceeding the budget and are removed.

The key property of score-proportional allocation is that the budget concentrates on a few high-scoring groups. The groups that matter for a given input obtain large $\hat{S}^{(g)}$ and thus budgets near the ceiling $c$. With $c$ set on the order of the context length, these groups can retain most of the context in 16-bit precision. As a result, the answer tokens in long-context retrieval, the needles that single-step $\Delta o$ selection can easily miss, are largely shielded from eviction within these groups. Most low-scoring groups, by contrast, fall toward the floor $r$, keeping the average per-group budget low.

In our setup, the ceiling $c$ is set to the reference prefill length of the workload. The target average budget $b$ is determined by the target compression ratio of each evaluation operating point. Unlike the streaming head partition of DuoAttention, which is fixed once by offline training, the per-group budgets are dynamically determined by the peek for each input. With this single allocation mechanism, \methodname{} covers the full range of workloads, from short- and medium-context QA and code generation to long-context retrieval.

%% file: sections/04-experiments.tex
\section{Experiments}
\label{sec:exp}

\subsection{Experimental Setup}
\begin{table*}[t]
  \centering
  \small
  \setlength{\tabcolsep}{4pt}
  \begin{tabular*}{\textwidth}{@{\extracolsep{\fill}}l c *{7}{c} *{3}{c} c@{}}
    \toprule
     & RULER & \multicolumn{7}{c}{LongBench} & \multicolumn{3}{c}{BFCL v3} & HumanEval \\
    \cmidrule(lr){2-2}\cmidrule(lr){3-9}\cmidrule(lr){10-12}\cmidrule(l){13-13}
    Category & All & SQA & MQA & Sum. & FS & Syn. & Code & All & simple & live\_simple & live\_multiple & -- \\
    \midrule
    \multicolumn{13}{@{}l}{\textit{(a) Llama-3.1-8B-Instruct}}\\
    \addlinespace[1pt]
    Full KV cache & 93.9 & 40.7 & 46.5 & 28.7 & 69.0 & 55.3 & 59.4 & 50.8 & 94.0 & 72.0 & 67.0 & 69.5 \\
    \cmidrule(lr){1-13}
    KV cache budget & \multicolumn{1}{c}{10\%} & \multicolumn{7}{c}{1.5\%} & \multicolumn{3}{c}{20\%} & 40\% \\
    {\small\itshape\quad$\approx$\,tokens} & \multicolumn{1}{c}{\small 1583} & \multicolumn{7}{c}{\small 139} & \multicolumn{3}{c}{\small 143} & {\small 72} \\
    \cmidrule(lr){1-13}
    H2O & 19.2 & 28.3 & 40.0 & 22.8 & 49.4 & 53.3 & 53.4 & 42.0 & 53.0 & 36.0 & 5.0 & 56.7 \\
    StreamingLLM & 19.0 & 23.1 & 35.9 & 18.3 & 39.3 & \textbf{55.4} & 57.9 & 39.7 & 21.0 & 4.0 & 0.0 & 15.9 \\
    SnapKV & 70.4 & 29.6 & 39.9 & 20.0 & 48.9 & 53.3 & 55.5 & 42.2 & 62.0 & 30.0 & 4.0 & 37.8 \\
    PyramidKV & 69.9 & 29.6 & 39.9 & 20.0 & 48.9 & 53.3 & 55.5 & 42.2 & 62.0 & 30.0 & 4.0 & 37.8 \\
    AdaKV & 81.0 & 32.1 & 44.5 & 20.8 & 55.4 & 54.3 & \textbf{60.8} & 46.0 & 86.0 & 53.0 & 22.0 & 53.0 \\
    HeadKV $^\dagger$ & 80.6 & 36.5 & 44.6 & 22.6 & 59.3 & 53.3 & 59.9 & 47.3 & 91.0 & 64.0 & 43.0 & 65.2 \\
    CriticalKV-HeadKV $^\dagger$ & 81.6 & 37.8 & \textbf{45.7} & 23.1 & 58.6 & 53.0 & 60.3 & 47.6 & \textbf{92.0} & \textbf{70.0} & \textbf{48.0} & 67.1 \\
    \textbf{\methodname{}} & \textbf{89.9} & \textbf{37.9} & 44.5 & \textbf{23.5} & \textbf{66.2} & 55.1 & 59.2 & \textbf{48.7} & \textbf{92.0} & 65.0 & 40.0 & \textbf{72.0} \\
    \midrule
    \multicolumn{13}{@{}l}{\textit{(b) Mistral-7B-Instruct-v0.2}}\\
    \addlinespace[1pt]
    Full KV cache & 85.8 & 30.8 & 26.1 & 27.0 & 67.0 & 45.0 & 54.2 & 43.0 & 82.0 & 51.0 & 56.0 & 42.1 \\
    \cmidrule(lr){1-13}
    KV cache budget & \multicolumn{1}{c}{10\%} & \multicolumn{7}{c}{1.5\%} & \multicolumn{3}{c}{20\%} & 40\% \\
    {\small\itshape\quad$\approx$\,tokens} & \multicolumn{1}{c}{\small 1604} & \multicolumn{7}{c}{\small 142} & \multicolumn{3}{c}{\small 152} & {\small 74} \\
    \cmidrule(lr){1-13}
    H2O & 18.8 & 20.3 & 17.1 & 22.1 & 53.6 & 22.4 & 47.2 & 32.8 & 42.0 & 22.0 & 0.0 & 20.7 \\
    StreamingLLM & 22.5 & 15.3 & 14.6 & 17.0 & 47.3 & 14.5 & 49.6 & 29.6 & 0.0 & 0.0 & 0.0 & 5.5 \\
    SnapKV & 55.4 & 22.4 & 18.9 & 19.5 & 51.3 & 31.7 & 49.1 & 34.1 & 5.0 & 0.0 & 1.0 & 11.0 \\
    PyramidKV & 53.9 & 22.4 & 18.9 & 19.5 & 51.3 & 31.7 & 49.1 & 34.1 & 5.0 & 0.0 & 1.0 & 12.2 \\
    AdaKV & 61.9 & 23.8 & 20.4 & 20.9 & 56.0 & 34.0 & 52.6 & 36.7 & 7.0 & 2.0 & 1.0 & 21.3 \\
    HeadKV $^\dagger$ & 67.1 & 26.9 & \textbf{21.8} & 22.2 & 64.3 & 33.8 & 52.3 & 38.8 & 69.0 & 38.0 & 27.0 & 35.4 \\
    CriticalKV-HeadKV $^\dagger$ & 74.9 & \textbf{27.5} & 20.8 & 22.6 & \textbf{65.4} & 33.4 & \textbf{52.9} & 39.1 & 68.0 & 41.0 & 31.0 & \textbf{36.6} \\
    \textbf{\methodname{}} & \textbf{75.9} & 25.8 & 21.3 & \textbf{23.0} & 64.2 & \textbf{39.4} & 52.1 & \textbf{39.2} & \textbf{78.0} & \textbf{47.0} & \textbf{37.0} & 31.1 \\
    \bottomrule
  \end{tabular*}
  \caption{Accuracy (\%) on public benchmarks. $^\dagger$Requires per-model offline profiling.}
  \label{tab:benchmatrix}
\end{table*}

\paragraph{Evaluation Protocol.} In deployment, \methodname{} keeps the tokens generated during the peek and continues from them, so the only cost is that compression is deferred by a few tokens. In evaluation, however, keeping this behavior would let the first few generated tokens be produced at full-cache quality, whereas the baselines generate every token from the compressed cache. This would bias the comparison in our favor. Therefore, in every accuracy evaluation in this paper, we take only the head scores and token scores from the peek, discard all tokens the peek generated, and score only the output regenerated from the end of the prefill with the compressed KV cache. All methods thus share the same decoding starting point.

\paragraph{Evaluation Benchmarks.}
When the target output is only one or two tokens, the one-token peek can effectively see the answer, inflating accuracy. We therefore exclude multiple-choice benchmarks and select benchmarks that require continuous generation on the compressed cache. We use four benchmarks: RULER-16k~\citep{ruler}, which measures long-context recall and QA; LongBench (English)~\citep{longbench}, which covers long-context QA, summarization, few-shot tasks, and code completion; BFCL v3~\citep{bfcl}, which evaluates tool calling; and HumanEval~\citep{humaneval}, a code generation benchmark.

For RULER-16k, we sample 100 examples per task across 13 tasks, 1,300 in total. For LongBench, we use all 3,750 examples of the 16 English tasks, following the official protocol with inputs truncated to 15,000 tokens. For BFCL, we sample 100 examples from each of the single-turn categories of v3 (\texttt{simple}, \texttt{live\_simple}, \texttt{live\_multiple}), 300 in total, and grade them by AST matching. For HumanEval, we use all 164 examples in a 0-shot setting and score by pass@1. All sampling uses seed 42, and all generation uses greedy decoding.

\paragraph{Baselines and Models.}
As baselines, we use the representative KV cache compression methods StreamingLLM~\citep{streamingllm}, H2O, SnapKV, and PyramidKV, the recent methods AdaKV, HeadKV (its reasoning head-score variant), and CriticalKV-HeadKV, which combines HeadKV with the output-perturbation rescoring of CriticalKV~\citep{criticalkv}.

We evaluate on Llama-3.1-8B-Instruct~\citep{llama3} and Mistral-7B-Instruct-v0.2~\citep{mistral}. Both are instruct models that support long contexts (128K and 32K, respectively).

\subsection{Accuracy across Memory Budgets}

We evaluate the accuracy of each method across memory budgets on RULER-16k (Figure~\ref{fig:rulerbudgetgrid}). The curves compare \methodname{} against AdaKV, HeadKV, and CriticalKV-HeadKV, the three baselines with the highest RULER accuracy in Table~\ref{tab:benchmatrix}. The remaining methods are compared only at the 10\% budget in that table. The x-axis is the measured physical KV cache size of each method, as a fraction of the full cache. On Llama-3.1-8B, \methodname{} achieves the highest accuracy across the entire 1.7--10\% range. At the 3\% budget, its 13-task average of 0.811 leads the runner-up AdaKV (0.701) by 11.0pp. On NIAH single-3 and multikey-3, where the baselines collapse to 0.00--0.05, \methodname{} maintains 1.000 and 0.680, respectively. On Mistral-7B, the advantage depends on the budget. CriticalKV-HeadKV leads by 3.0pp and 3.8pp at 1.7\% and 3\%, the two methods tie at 7\%, and \methodname{} leads at 10\% and 30\%.

The FWE (Frequent Words Extraction) task at low budgets also exposes a weakness of \methodname{}. We analyze the cause in detail in Section~\ref{sec:exp-ablation}.

\subsection{Accuracy across Benchmarks}

We evaluate on the four benchmarks RULER-16k, LongBench, BFCL v3, and HumanEval, which together cover long-context recall, QA, summarization, tool calling, and code generation (Table~\ref{tab:benchmatrix}). The budget for each benchmark is the measured physical KV cache size as a fraction of the full cache. It counts everything retained from the prefill cache and excludes the KV of generated tokens. The budget level is chosen separately for each benchmark to avoid ceiling and floor effects. \methodname{} stays closest to the full-cache accuracy on RULER and LongBench, while on BFCL and HumanEval the advantage depends on the model. It leads StreamingLLM, H2O, SnapKV, and PyramidKV, which evict tokens without any head-level allocation, on all four benchmarks. Against AdaKV, HeadKV, and CriticalKV-HeadKV, which allocate the budget unevenly across heads, the margins narrow, yet \methodname{} still achieves the highest accuracy in more than half of the per-model columns of Table~\ref{tab:benchmatrix}, without the per-model offline profiling that the HeadKV and CriticalKV-HeadKV require. HeadKV and CriticalKV-HeadKV rely on per-model offline profiling. Among methods that require no offline profiling, \methodname{} achieves the highest accuracy in most columns. \footnote{Results at additional budget levels are provided in the supplementary material.}

\subsection{Ablation Studies}
\label{sec:exp-ablation}
We showed earlier in Figure~\ref{fig:headworkload} that the content@1 signal generally predicts decode-time head importance better than offline profiles or prefill signals. We now analyze where the accuracy gain of \methodname{} comes from (Table~\ref{tab:signalaxes}). Figure~\ref{fig:peekwindow} then shows how accuracy varies with the kind of token the peek signal is read from and with the number of peek tokens.

We first decompose where the accuracy gain of \methodname{} comes from along its two axes, head selection and token selection. As Table~\ref{tab:signalaxes} shows, replacing the peek token scores with prefill signals costs 0.6pp on RULER and 5.5pp on HumanEval. When the head scores are replaced as well, the loss grows to 11.2pp and 9.8pp. With both axes read from the prefill stage, this configuration falls into the same family as AdaKV, which allocates per-head budgets from prefill attention — though the two differ in their scoring functions and allocation rules.

\begin{table}[t]
  \centering
  \small
  \setlength{\tabcolsep}{3pt}
  \begin{tabular}{@{}l cc cc@{}}
    \toprule
    & \multicolumn{2}{c}{RULER (10\%)} & \multicolumn{2}{c}{HumanEval (40\%)} \\
    \cmidrule(lr){2-3}\cmidrule(l){4-5}
    & Acc. & $\Delta$ & Acc. & $\Delta$ \\
    \midrule
    \methodname{} & \textbf{89.9} & & \textbf{72.0} & \\
    \quad + token signal $\to$ prefill & 89.3 & $-$0.6 & 66.5 & $-$5.5 \\
    \quad + head signal $\to$ prefill & 78.7 & $-$11.2 & 62.2 & $-$9.8 \\
    \bottomrule
  \end{tabular}
  \caption{Accuracy (\%) as the peek signals are cumulatively replaced by prefill signals.}
  \label{tab:signalaxes}
\end{table}

Figure~\ref{fig:peekwindow} compares three positions for reading the peek signal. Prefill-last pools attention over the last 32 prefill tokens. Decode peek observes the first $N$ decode tokens as they are generated. Content peek skips format tokens and observes $N$ content tokens. Here $N$ is the peek window size. With a single token ($N{=}1$), Content peek is the most accurate (0.811 on RULER at 3\%), and Decode peek (0.610) falls even below Prefill-last (0.664). The same ordering holds on LongBench (Figure~\ref{fig:peekwindow}(c)). The benefit of skipping format tokens persists as the budget grows. The gap between Content peek and Decode peek is 20.1pp at the 3\% budget and 11.5pp at 10\%. Even though Prefill-last pools 32 tokens, it trails Content peek by 14.7pp at 3\% and 11.2pp at 10\%. As $N$ grows, however, the Decode peek window passes the format span and begins to include content tokens, so its accuracy converges to that of Content peek (0.893 vs 0.900 at $N{=}8$, RULER 3\%).

Next, the effect of the number of peek tokens depends on the budget. At the low 3\% budget on RULER-16k, the one-token peek is less accurate than multi-token peeks. About 95\% of the +6.2pp recovered by a second peek token comes from the aggregation tasks FWE and CWE (Common Words Extraction) and from NIAH multikey-3. A per-item error analysis indicates that the loss stems not from misselecting heads but from limited coverage in token selection. Even the one-token peek locates the answer correctly, but it fails to retain the tails of multi-token spans, such as the trailing word tokens in FWE. The gain from widening the observation to eight tokens shrinks from 8.9pp at the 3\% budget to 2.7pp at 10\%, and on LongBench it stays within 0.7pp even at the 1.5\% budget. The deficit is therefore confined to recall and aggregation tasks at low budgets, and we adopt the one-token peek as the default.

\begin{figure}[t]
  \centering
  \includegraphics[width=\columnwidth]{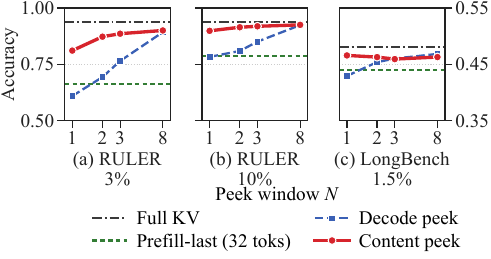}
  \caption{Accuracy by signal-reading position and the number of peek tokens.}
  \label{fig:peekwindow}
\end{figure}

%% file: sections/06-conclusion.tex
\section{Conclusion}
\label{sec:conclusion}

In this paper, we found that the active heads shift between the prefill and decode stages even within a single input, whereas the active-head composition remains stable across content tokens once format tokens are filtered out on the decode path. Building on this observation, \methodname{} jointly obtains head-group scores and token-importance scores from a single observation of the first content token — the one-token peek — and allocates KV budgets anew for each input without offline profiling or training. Across the four public benchmarks, \methodname{} achieves the highest accuracy in more than half of the per-model comparison columns, and on RULER and LongBench, including long-context retrieval, it stays closest to the full-cache accuracy at small memory budgets. As future work, we leave improving performance in multi-turn settings, a weakness reported to be shared by KV cache eviction approaches~\citep{scbench}.

%% file: supplementary/sections/A-proofs.tex
\section{Detailed Derivations for Section 3}
\label{supp:proofs}

\subsection{Setup and Notation}
\label{supp:setup}

We adopt the notation of Section~3.1 in the main paper, fix a single layer and the observation step $t^\ast$, and omit the step subscript following the convention of Section~3.3 in the main paper. With $e_i^{(h)} = \exp\big(q^{(h)\top} k_i^{(g(h))}/\sqrt{d_h}\big)$, the attention weights and the output of head $h$ are as in Eqs.~(2) and~(3) in the main paper:
\begin{equation}
\label{eq:supp-setup}
\alpha_i^{(h)} = \frac{e_i^{(h)}}{\sum_{i' \le t^\ast} e_{i'}^{(h)}}, \qquad o^{(h)} = \sum_{i \le t^\ast} \alpha_i^{(h)} v_i^{(g(h))}.
\end{equation}
Evicting token $j$ means removing the entry $(k_j, v_j)$ from the KV cache. The unnormalized weights $e_i^{(h)}$ of the remaining tokens ($i \ne j$) are unchanged, since their keys and the query are untouched. Only the term $e_j^{(h)}$ drops out of the softmax denominator. Hereafter we omit the head superscript $(h)$ and the group superscript $(g(h))$ where no confusion arises.

\subsection{Exact Output Perturbation under Single-Token Eviction (Eq.~(6) in the Main Paper)}
\label{supp:deltao}

\begin{lemma}[Softmax renormalization]
\label{lem:renorm}
If $\alpha_j < 1$, the attention weights $\tilde{\alpha}_i$ after evicting token $j$ satisfy, for all $i \le t^\ast$ with $i \ne j$,
\begin{equation}
\label{eq:supp-renorm}
\tilde{\alpha}_i = \frac{\alpha_i}{1 - \alpha_j}.
\end{equation}
\end{lemma}

\begin{proof}
Let $Z = \sum_{i' \le t^\ast} e_{i'}$, so that $\alpha_i = e_i / Z$. After the eviction the denominator becomes
\begin{equation}
Z' = \sum_{i' \ne j} e_{i'} = Z - e_j = Z(1 - \alpha_j)
\end{equation}
and the numerators $e_i$ are unchanged, so
\begin{equation}
\tilde{\alpha}_i = \frac{e_i}{Z'} = \frac{e_i}{Z(1-\alpha_j)} = \frac{\alpha_i}{1 - \alpha_j}. \qedhere
\end{equation}
\end{proof}

\begin{proposition}[Exact single-token perturbation; Eq.~(6)]
\label{prop:deltao}
The change $\Delta o_j = \tilde{o} - o$ that evicting token $j$ induces in the attention output at the observation step is, exactly and without approximation,
\begin{equation}
\label{eq:supp-deltao}
\Delta o_j = \frac{\alpha_j}{1 - \alpha_j}\,\big(o - v_j\big).
\end{equation}
\end{proposition}

\begin{proof}
By Lemma~\ref{lem:renorm}, the output after the eviction is
\begin{equation}
\label{eq:supp-otilde}
\begin{split}
\tilde{o} = \sum_{i \ne j} \tilde{\alpha}_i\, v_i
&= \frac{1}{1-\alpha_j} \sum_{i \ne j} \alpha_i v_i \\
&= \frac{1}{1-\alpha_j}\Big( \sum_{i \le t^\ast} \alpha_i v_i - \alpha_j v_j \Big) = \frac{o - \alpha_j v_j}{1 - \alpha_j}.
\end{split}
\end{equation}
Therefore
\begin{equation}
\begin{split}
\Delta o_j = \tilde{o} - o
&= \frac{o - \alpha_j v_j - (1 - \alpha_j)\, o}{1 - \alpha_j} \\
&= \frac{\alpha_j\, o - \alpha_j v_j}{1 - \alpha_j} = \frac{\alpha_j}{1 - \alpha_j}\,\big(o - v_j\big). \mbox{\qedhere}
\end{split}
\end{equation}
\end{proof}

The condition $\alpha_j < 1$ holds for every eviction candidate (a prefill token) $j \le t^\ast - 1$, since causal attention always assigns a positive weight $\alpha_{t^\ast} > 0$ to the current token $t^\ast$ (for numerical stability, the implementation clamps the denominator to $\max(1-\alpha_j, 10^{-6})$).

\subsection{Residual-Stream Score (Eq.~(7) in the Main Paper)}
\label{supp:tokscore}

\begin{proposition}[Residual-stream perturbation; Eq.~(7)]
\label{prop:tokscore}
Under the per-head decomposition $\mathrm{Attn}(x) = \sum_{h} W_O^{(h)} o^{(h)}$ of Eq.~(4) in the main paper, the $\ell_2$ magnitude of the change that evicting token $j$ from head $h$ induces in the attention block's output (the quantity added to the residual stream) is
\begin{equation}
\label{eq:supp-tokscore}
\begin{split}
s_j^{(h)} &= \big\lVert W_O^{(h)} \Delta o_j^{(h)} \big\rVert_2 \\
&= \frac{\alpha_j^{(h)}}{1 - \alpha_j^{(h)}}\, \big\lVert W_O^{(h)} \big( o^{(h)} - v_j^{(g(h))} \big) \big\rVert_2.
\end{split}
\end{equation}
\end{proposition}

\begin{proof}
If token $j$ is evicted from head $h$ only, the other terms $W_O^{(h')} o^{(h')}$ ($h' \ne h$) of the decomposition are unchanged, and by linearity the change in the attention block's output is
\begin{multline}
\sum_{h'} W_O^{(h')} \tilde{o}^{(h')} - \sum_{h'} W_O^{(h')} o^{(h')} \\
= W_O^{(h)} \big( \tilde{o}^{(h)} - o^{(h)} \big) = W_O^{(h)} \Delta o_j^{(h)}.
\end{multline}
Substituting the identity of Proposition~\ref{prop:deltao}, and using $\lVert c\, W_O (o - v_j) \rVert_2 = c\, \lVert W_O (o - v_j) \rVert_2$ for the scalar $c = \alpha_j / (1 - \alpha_j) > 0$ (absolute homogeneity of the norm), we obtain the second equality of Eq.~\eqref{eq:supp-tokscore}. \qedhere
\end{proof}

Since KV entries are shared at the group level (Section~3.1 in the main paper), physically evicting token $j$ from group $g$ perturbs all of its query heads at once. Each head's output change is still given by Proposition~\ref{prop:deltao}, so by linearity the change in the attention block's output is exactly $\sum_{h:\, g(h)=g} W_O^{(h)} \Delta o_j^{(h)}$. The sum of per-head scores $s_j^{(g)} = \sum_{h:\, g(h)=g} s_j^{(h)}$, which Section~3.3 in the main paper uses as the group token score, upper-bounds the $\ell_2$ magnitude of this change by the triangle inequality. A token evicted for a low score is therefore guaranteed to induce a perturbation no larger than that score. Computing the exact magnitude with the Gram expansion of Section~\ref{supp:gram} would require the cross-Gram matrices $W_O^{(h)\top} W_O^{(h')}$ for every pair of heads in the group, multiplying the per-token cost by the group size. In multi-head attention, where the group size is one, the two quantities coincide.

\subsection{Gram-Matrix Expansion and Computational Cost}
\label{supp:gram}

\begin{lemma}[Gram expansion]
\label{lem:gram}
Let $G^{(h)} = W_O^{(h)\top} W_O^{(h)} \in \mathbb{R}^{d_h \times d_h}$ (positive semi-definite). Then $\lVert W_O^{(h)} z \rVert_2^2 = z^\top G^{(h)} z$ for any $z \in \mathbb{R}^{d_h}$, and in particular, for $z = o - v_j$,
\begin{equation}
\label{eq:supp-gram}
\big\lVert W_O^{(h)} (o - v_j) \big\rVert_2^2 = o^\top G^{(h)} o - 2 \big(G^{(h)} o\big)^\top v_j + v_j^\top G^{(h)} v_j.
\end{equation}
\end{lemma}

\begin{proof}
We have $\lVert W_O z \rVert_2^2 = (W_O z)^\top (W_O z) = z^\top W_O^\top W_O\, z = z^\top G z$. Substituting $z = o - v_j$ and expanding the quadratic form yields Eq.~\eqref{eq:supp-gram}, where the symmetry of $G$ ($G^\top = G$) combines the two cross terms, $o^\top G v_j = v_j^\top G o$. \qedhere
\end{proof}

The three terms of Eq.~\eqref{eq:supp-gram} have different reuse structures. (i) $G^{(h)}$ is a function of the model weights alone, so it is precomputed once per head ($O(d\, d_h^2)$, offline). (ii) $v_j^\top G^{(h)} v_j$ does not depend on the query, so it is computed once per token and reused when there are multiple observation steps ($O(d_h^2)$ per token). (iii) $o^\top G^{(h)} o$ and the vector $u = G^{(h)} o$ are computed once per step ($O(d_h^2)$ per step), and the cross term is the inner product $u^\top v_j$ at $O(d_h)$ per token. The cost per candidate token thus drops from the $O(d\, d_h)$ required by an explicit projection with $W_O^{(h)} \in \mathbb{R}^{d \times d_h}$ to at most $O(d_h^2)$, a reduction by a factor of $d / d_h$ ($d \gg d_h$). The norm $\lVert W_O^{(h)} (o - v_j) \rVert_2$ is obtained by computing its square with Eq.~\eqref{eq:supp-gram} and taking a single square root at the end. The score $s_j^{(h)}$ then follows by multiplying this norm by $\alpha_j/(1-\alpha_j)$.

The measured gains of this expansion appear in both computation time and working memory. Token scoring speeds up by a factor of 6.66 ($N=1024$) to 12.44 ($N=16384$) depending on the number of candidate tokens $N$, and working memory shrinks by a factor of 17.5 ($N=16384$). The speed gain falls short of the theoretical factor $d/d_h = 32$ ($d=4096$, $d_h=128$) because the matrix multiplications on the Gram path are small and run at lower hardware efficiency. The measurement conditions and the effect on the full pipeline are covered in Section~\ref{supp:results:overhead}, and Table~\ref{tab:supp-overhead} summarizes the resulting decode throughput and KV memory.

\begin{table}[t]
  \centering
  \small
  \setlength{\tabcolsep}{4pt}
  \begin{tabular}{rccrcr}
    \toprule
    & \multicolumn{2}{c}{decode (ms/token)} & & & \\
    \cmidrule(lr){2-3}
    Context & dense & \methodname{} & gain & attn & KV (MiB) \\
    \midrule
     2{,}048 & 21.48 & 21.18 & 1.4\% & 1.38$\times$ & 256 $\to$ 25.6 \\
     4{,}096 & 21.81 & 21.23 & 2.7\% & 1.69$\times$ & 512 $\to$ 51.2 \\
     8{,}192 & 22.52 & 21.38 & 5.1\% & 2.21$\times$ & 1024 $\to$ 102.4 \\
    15{,}715 & 23.54 & 21.60 & 8.3\% & 2.64$\times$ & 1964 $\to$ 196.5 \\
    \bottomrule
  \end{tabular}
  \caption{Decode throughput and KV memory after compression. \emph{attn} is the speedup of the attention kernel alone, and \emph{gain} is the time reduction of the whole decode step.}
  \label{tab:supp-overhead}
\end{table}

%% file: supplementary/sections/B-implementation.tex
\section{Implementation Details}
\label{supp:impl}

\subsection{Format-Token Judgment Rules}
\label{supp:formatrules}

The one-token peek of \methodname{} (Section~3.2 in the main paper) identifies format tokens by the syntactic structure of the partially generated text rather than by a fixed word list. The classifier builds a character-level format mask over the text and projects it onto token boundaries. A token is content if the characters it covers include at least one alphanumeric character not marked as format, and format if they include none, so whitespace and punctuation tokens automatically count as format.

Three layers fill the mask. The document-structure layer marks a code-fence line such as \texttt{\textasciigrave\textasciigrave\textasciigrave python} as format in its entirety, including the language tag, and marks list markers and header markers at the start of a line as format while leaving what follows them as content. A region opened by a fence remains a code region even before the fence is closed, and it extends to the end of the text generated so far. Even without a fence, the whole text is treated as a code region when its first line starts with a declaration such as \texttt{def} or \texttt{import}. The prose function-word layer marks the articles \texttt{the}, \texttt{a}, and \texttt{an} outside code regions as format, which keeps the signal from landing on an article. The code-syntax layer sweeps code regions with a parser, marking only keywords, operators, and delimiters as format and leaving identifiers, literals, strings, and comments as content. In \texttt{import numpy} the only format token is \texttt{import}, so the peek stops at \texttt{numpy}, and in a function definition only \texttt{def} is format, so the function name and its arguments remain content.

This layer uses the tree-sitter parser. We cap the length the peek generates on its own (32 tokens in our setting). If no content token appears before the cap is reached, the peek falls back to the signal at the first decode step, with no format skipping. Per-benchmark peek step counts are given in Table~\ref{tab:supp-peeksteps}.

%% file: supplementary/sections/C-reproducibility.tex
\section{Reproducibility}
\label{supp:repro}

\subsection{Hyperparameters}
\label{supp:repro:hparams}

The configuration of \methodname{} splits into a part shared across all benchmarks and a part set by the operating point. The shared part has three components. The peek stops at the first content token it meets, which takes a median of 2 extra steps (Table~\ref{tab:supp-peeksteps}), and the search is capped at 32 tokens for the case where no content token appears (Section~\ref{supp:formatrules}). The recent window of 8 tokens at the tail of the prefill is always kept. The per-group budget floor is $r = 8$. In addition, the length $S$ of the sink span excluded from head scoring does not vary across benchmarks either, but it differs between the models. Llama uses the constant 8, Mistral computes it per item, and the rule and its rationale follow below. The part set by the operating point comprises the target average budget $b$ and the ceiling $c$ of Section~3.4 in the main paper, and we set $c$ to the reference prefill length of the workload. Table~\ref{tab:supp-hparams} gives the values that produced each cell of Table~2 in the main paper.

\begin{table}[t]
  \centering
  \small
  \begin{tabular}{@{}llrrrc@{}}
    \toprule
    Model & Benchmark & $b$ & $c$ & $r$ & $S$ \\
    \midrule
    \multirow{4}{*}{Llama-3.1-8B} & RULER-16k & 1581 & 15689 & 8 & \multirow{4}{*}{8} \\
                                  & LongBench & 131 & 8878 & 8 & \\
                                  & BFCL v3 & 150 & 716 & 8 & \\
                                  & HumanEval & 85 & 184 & 8 & \\
    \midrule
    \multirow{4}{*}{Mistral-7B} & RULER-16k & 1607 & 15854 & 8 & \multirow{4}{*}{per item} \\
                                & LongBench & 134 & 9759 & 8 & \\
                                & BFCL v3 & 169 & 765 & 8 & \\
                                & HumanEval & 96 & 186 & 8 & \\
    \bottomrule
  \end{tabular}
  \caption{Per-benchmark hyperparameters of \methodname{}. $b$ is the target average budget, $c$ the ceiling, $r$ the per-group floor on kept tokens, and $S$ the sink span length.}
  \label{tab:supp-hparams}
\end{table}

\begin{figure*}[t]
  \centering
  \includegraphics[width=\textwidth]{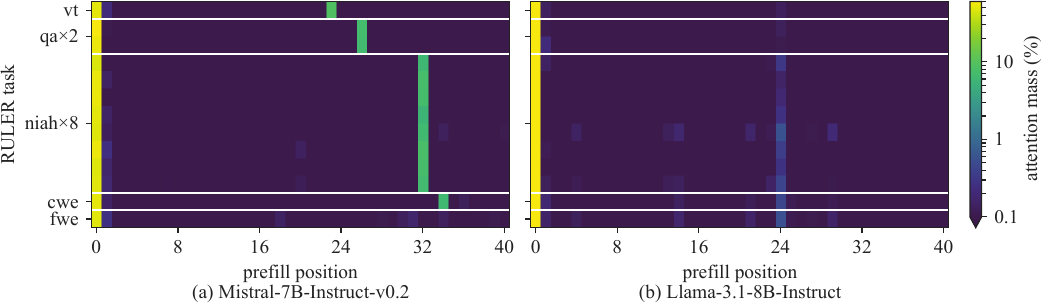}
  \caption{Attention mass over prefill positions for the 13 RULER tasks. Panel (a) is Mistral and (b) is Llama, the color scale is logarithmic, and the darkest cells are below 0.1\%.}
  \label{fig:sinktwopeak}
\end{figure*}

\subsubsection{Determining the Sink Span}
The head score of Eq.~(5) in the main paper is computed over the prefill positions after removing the attention-sink span at the front. The rule that sets the length $S$ of this span differs between the two models. For Llama it is the constant $S = 8$. Mistral uses no constant and computes the span per item, setting $S = j + 1$ where $j$ is the index of the first prompt token that contains a newline, and excluding the first $S$ positions. That is, everything through the last token of the first line, line break included, drops out of the signal candidates. If no newline appears within the first 64 tokens, the rule falls back to the Llama value of 8.

This rule exists because Mistral places a second attention sink in addition to the BOS. Figure~\ref{fig:sinktwopeak} shows this. We measure the attention mass over prefill positions at the step where the peek actually reads its signal, and sort the rows by the position of the first newline in each task. In both models the BOS at position 0 is the brightest cell. The second-brightest cell, however, appears only in Mistral, and its position shifts from 23 to 34 across tasks.

Task format, not item content, fixes the position of this peak. Within a task the peak sits at the same place across items, because RULER opens every prompt of a task with a fixed instruction that is identical up to the first newline regardless of the item. That place is the single cell right before the first newline. In 12 of the 13 RULER tasks this one cell takes 5.5\% to 8.9\% of the attention mass. The remaining task has its only newline at the tail of the prompt, so no peak appears near the front, and it is the bottom row of the figure. In Llama even the largest cell over the same span stays at background level, 0.13\% to 0.59\%, so the ranges of the two models do not overlap. Since the head score is the maximum over this span, leaving this one cell in the signal would let prompt format rather than input content decide the head ranking.

We ran the same probe on Mistral-7B-Instruct-v0.3 and on the Llama-3.1-8B base model, and the results matched those of each family's instruct model (4.2\% to 7.1\% for v0.3 and 0.04\% to 0.18\% for the base model). The property is therefore neither specific to Mistral v0.2 nor produced by instruction tuning.

This boundary applies only to the ranking signal and plays no part in the mask that decides retention and eviction. Sink tokens are not erased from the cache but only dropped from the candidates when scoring heads, so the memory accounting does not change.

\begin{figure*}[t]
  \centering
  \includegraphics[width=\textwidth]{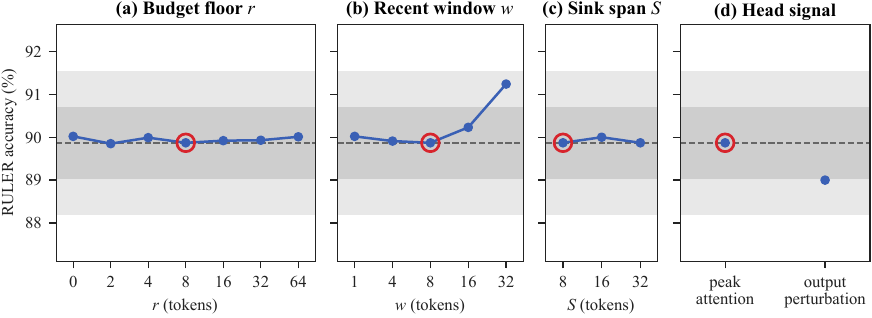}
  \caption{Hyperparameter sensitivity on RULER at a 10\% budget. Rings mark the defaults and bands are $\pm1$/$\pm2$ binomial standard errors around them.}
  \label{fig:hparamsens}
\end{figure*}

\subsubsection{Sensitivity to Hyperparameters}
We measure how sensitive \methodname{} is to the hyperparameters of Section~\ref{supp:repro:hparams} on RULER-16k with Llama-3.1-8B at the 10\% budget, varying one hyperparameter at a time around the operating point (Figure~\ref{fig:hparamsens}). The budget floor $r$ and the sink span $S$ matter least. Sweeping $r$ from 0 to 64 moves the 13-task average by far less than the binomial standard error of the benchmark. The realized memory also stays essentially unchanged, because the slope of the score-proportional allocation, Eq.~(8) in the main paper, compensates for $r$. Even removing the floor altogether costs nothing, which bears out the claim of Section~3.2 in the main paper that low-scoring groups remain adequate with only the recent window and a few top-scoring tokens. The sink span behaves the same way, since it only drops candidates from the head score and never touches the retention mask. Only the recent window $w$ moves the result, and even there the gain remains modest. We nevertheless keep both $w$ and $r$ at small defensive defaults rather than tuned ones, because at the low budgets of the other workloads larger settings could leave too little budget for the heads the peek identifies. Replacing the peak-attention head signal of Eq.~(5) in the main paper with the renormalized output perturbation, the same machinery the token scores use, does not help either.

\subsection{Computing Infrastructure}
\label{supp:repro:infra}

\begin{table}[t]
  \centering
  \small
  \begin{tabular}{@{}l p{0.27\columnwidth} p{0.27\columnwidth}@{}}
    \toprule
    Component & Cluster & Workstation \\
    \midrule
    Role & Benchmark runs & Overhead measurement \\
    \midrule
    GPU & RTX A6000 (48\,GB) & RTX 3090 Ti (24\,GB) \\
    CPU & AMD EPYC 9354 / 7413, Intel Xeon Gold 6348 & Intel Core Ultra 7 265K \\
    System memory & 503\,GB & 62\,GB \\
    Operating system & Linux 5.14 & Linux 7.1 \\
    Python & 3.12.8 & 3.14.6 \\
    PyTorch (CUDA) & 2.12.1 (12.6) & 2.11.0 (12.8) \\
    Transformers & 5.13.0 & 5.14.1 \\
    \bottomrule
  \end{tabular}
  \caption{Specifications of the two execution environments used for the experiments.}
  \label{tab:supp-infra}
\end{table}

The experiments ran in two environments with different roles, and Table~\ref{tab:supp-infra} lists them side by side. Benchmark runs took place on a shared GPU cluster. CPUs and system memory differ across its nodes, but the ones our experiments ran on share the same GPU and software stack. A single run loads the whole model onto one GPU and does not shard it across devices. The throughput and memory measurements of Section~\ref{supp:results:overhead} were taken separately on a standalone workstation, not on the cluster.

\subsection{Determinism and Number of Runs}
\label{supp:repro:runs}

All generation uses greedy decoding, and the numbers the paper reports come from a single run per condition. Every benchmark experiment in the paper ran through a single harness built on PyTorch and Hugging Face Transformers.

%% file: supplementary/sections/D-results.tex
\section{Additional Experimental Results}
\label{supp:results}

\subsection{Results at Additional Budget Levels}
\label{supp:results:budget}

Table~2 in the main paper reports each benchmark at a single budget point. To check that this point was not chosen in our favor, we extend the budget range on three benchmarks, adding 3\% and 5\% on LongBench, 30\% and 40\% on BFCL, and 50\% on HumanEval. The first point of each benchmark is the operating point of Table~2, carried over without re-measurement. RULER is not repeated here because Figure~4 in the main paper already plots a six-point curve.

The table and the figure show the same data. The BFCL value is the average of the three per-category values that the main paper lists separately, that is, the accuracy over all 300 items. In the figure, solid lines denote \methodname{}, dashed lines denote the methods that allocate the budget unevenly across heads, and dotted lines denote the methods that evict tokens without head-level allocation. The vertical axis range differs across panels because a shared axis would flatten the LongBench curves.

\begin{figure*}[t]
  \centering
  \includegraphics[width=\textwidth]{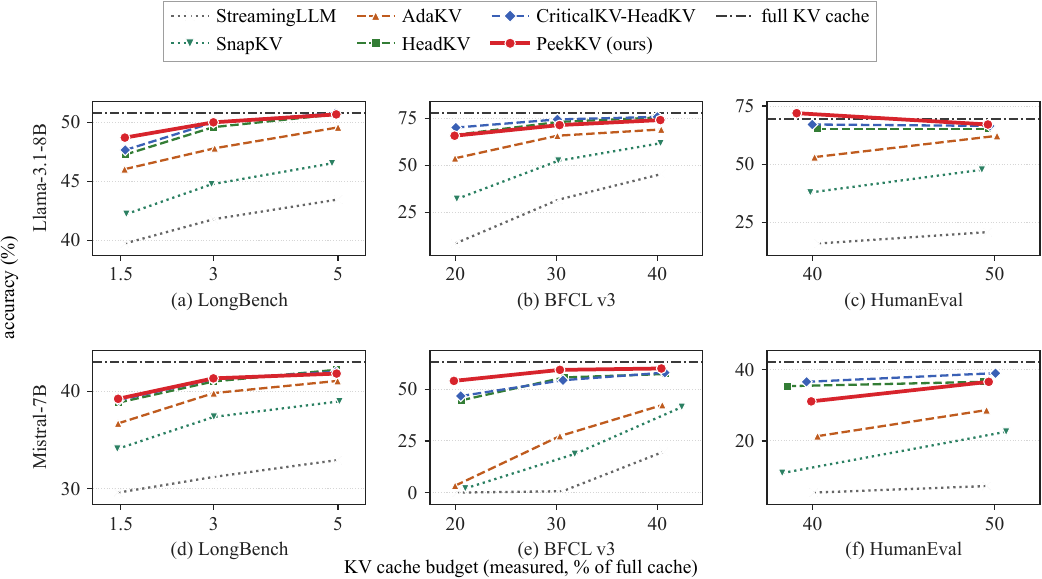}
  \caption{Accuracy as the budget varies. The horizontal axis is the \emph{measured} KV cache size rather than the target budget.}
  \label{fig:suppbudgetgrid}
\end{figure*}

\begin{table*}[t]
  \centering
  \small
  \setlength{\tabcolsep}{4pt}
  \begin{tabular*}{\textwidth}{@{\extracolsep{\fill}}l ccc c ccc c cc@{}}
    \toprule
    & \multicolumn{3}{c}{LongBench} & & \multicolumn{3}{c}{BFCL v3} & & \multicolumn{2}{c}{HumanEval} \\
    \cmidrule(lr){2-4}\cmidrule(lr){6-8}\cmidrule(lr){10-11}
    KV cache budget & 1.5\% & 3\% & 5\% & & 20\% & 30\% & 40\% & & 40\% & 50\% \\
    \midrule
    \multicolumn{11}{@{}l}{\textit{(a) Llama-3.1-8B-Instruct}}\\
    \addlinespace[1pt]
    Full KV cache & \multicolumn{3}{c}{50.8} & & \multicolumn{3}{c}{77.7} & & \multicolumn{2}{c}{69.5} \\
    \cmidrule(lr){1-11}
    StreamingLLM & 39.7 & 41.8 & 43.5 & & 8.3 & 31.7 & 45.3 & & 15.9 & 20.7 \\
    SnapKV & 42.2 & 44.7 & 46.5 & & 32.0 & 52.3 & 61.7 & & 37.8 & 47.6 \\
    AdaKV & 46.0 & 47.8 & 49.6 & & 53.7 & 65.7 & 69.0 & & 53.0 & 62.2 \\
    HeadKV$^\dagger$ & 47.3 & 49.6 & 50.7 & & 66.0 & 73.0 & 75.3 & & 65.2 & 65.2 \\
    CriticalKV-HeadKV$^\dagger$ & 47.6 & 50.0 & 50.8 & & 70.0 & 74.3 & 75.7 & & 67.1 & 66.5 \\
    \methodname{} & 48.7 & 50.0 & 50.7 & & 65.7 & 71.3 & 74.0 & & 72.0 & 67.1 \\
    \midrule
    \multicolumn{11}{@{}l}{\textit{(b) Mistral-7B-Instruct-v0.2}}\\
    \addlinespace[1pt]
    Full KV cache & \multicolumn{3}{c}{43.0} & & \multicolumn{3}{c}{63.0} & & \multicolumn{2}{c}{42.1} \\
    \cmidrule(lr){1-11}
    StreamingLLM & 29.6 & 31.2 & 32.9 & & 0.0 & 0.7 & 19.7 & & 5.5 & 7.3 \\
    SnapKV & 34.1 & 37.4 & 39.0 & & 2.0 & 18.7 & 41.3 & & 11.0 & 22.6 \\
    AdaKV & 36.7 & 39.8 & 41.1 & & 3.3 & 27.3 & 42.3 & & 21.3 & 28.7 \\
    HeadKV$^\dagger$ & 38.8 & 41.0 & 42.2 & & 44.7 & 55.7 & 57.3 & & 35.4 & 36.6 \\
    CriticalKV-HeadKV$^\dagger$ & 39.1 & 41.2 & 42.0 & & 46.7 & 54.3 & 58.0 & & 36.6 & 39.0 \\
    \methodname{} & 39.2 & 41.3 & 41.8 & & 54.0 & 59.3 & 60.0 & & 31.1 & 36.6 \\
    \bottomrule
  \end{tabular*}
  \caption{Accuracy (\%) across budget levels. $^\dagger$Requires offline profiling.}
  \label{tab:supp-budget}
\end{table*}

Figure~\ref{fig:suppbudgetgrid} and Table~\ref{tab:supp-budget} show that the standings observed at the operating points change little as the budget grows by a factor of two or three. The one exception is the 5\% budget on LongBench, the only point at which \methodname{} moves from leading to trailing, and there the top three methods are indistinguishable. On LongBench, \methodname{} scores highest on both models at the tightest budget. As the budget grows, the three methods with head-level allocation converge, while the methods without it do not catch up. A per-item paired bootstrap supports this reading. \methodname{} is significantly ahead of AdaKV, SnapKV, and StreamingLLM at all three budgets on both models, whereas against the two HeadKV-family methods the difference is significant only at 1.5\% on Llama, where \methodname{} leads. Once the budget is ample, the gap to the offline-profiled methods closes, while the gap to AdaKV, which allocates its budget from a prefill signal, remains.

On BFCL and HumanEval the advantage splits by model, as the main paper describes, and doubling the budget does not overturn the split. On BFCL, \methodname{} leads on Mistral and CriticalKV-HeadKV leads on Llama at all three budgets, and both gaps shrink as the budget grows. At 20\% on Mistral, StreamingLLM and SnapKV score near zero, and AdaKV collapses with them despite allocating its budget across heads. Head-level allocation alone is therefore not enough, and the signal that drives the allocation matters as well.

Across the table as a whole, \methodname{} leads every method that requires no offline profiling at every budget, and holds a narrow edge over the two offline-profiled methods, leading in more of the comparisons than it trails.

\subsection{Budget Allocation across Heads}
\label{supp:results:alloc}

\begin{figure*}[t]
  \centering
  \includegraphics[width=\textwidth]{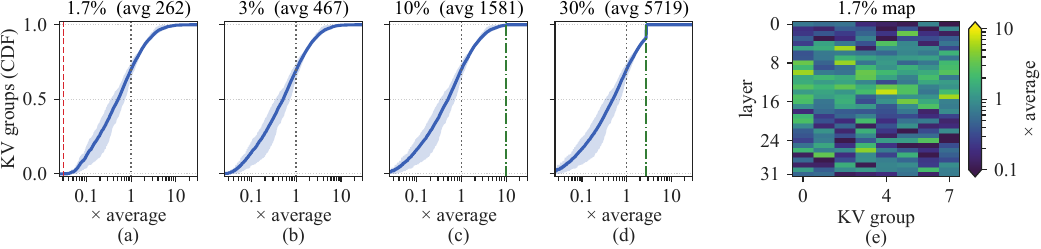}
  \caption{The budget \methodname{} assigns to each KV group on RULER-16k, as a multiple of the target average budget. Dashed lines mark the floor, dotted lines the uniform share, and dash-dotted lines the ceiling.}
  \label{fig:peekalloc}
\end{figure*}

Figure~\ref{fig:peekalloc} shows the actual shape of the budget that \methodname{} distributes across the head axis. The unit of allocation is the KV group, 256 for Llama-3.1-8B, and each group receives its budget $c_g$ through the score-proportional allocation of Section~3.4 in the main paper. A fixed set of RULER-16k items is shown at four of the budget levels of Figure~4 in the main paper. Panels (a) through (d) show the per-budget CDFs, and panel (e) maps the allocation onto the layer-and-group grid. All values are multiples of the target budget $b$; in the map, brighter cells receive more than the uniform share and darker cells less. The four allocations derive from a signal measured once, since the peek precedes compression and the budget enters only the allocation stage that follows. Each allocation matches the per-item memory of the run behind the corresponding point of Figure~4 to within 0.33\%.

The allocation departs sharply from uniform. The top 10\% of groups take 44\% of the budget, and more than half of the groups fall below half of the uniform share. The ranking of the groups shifts from item to item, a variation traced by the pale 10--90\% bands around the CDFs.

At 1.7\%, 3\%, and 10\% this shape merely scales with the budget, and the group ranking stays unchanged. What changes the shape is the ceiling $c$. At 30\% the ceiling binds for 8.7\% of the groups and flattens the concentration. The truncated share is not redistributed, so the realized mean falls to 82.7\% of the target. This flattening points in the same direction as Figure~4 in the main paper, where the gaps between methods narrow as the budget grows. A budget near the ceiling leaves little room to divide it differently across heads.

\subsection{Computational Overhead and Throughput}
\label{supp:results:overhead}

\methodname{} lets the model continue generating briefly before the budget is set and computes its scores at that point. This section measures what the procedure costs and what it returns, split into the computational cost of the selection step and the decode throughput after compression.

\subsubsection{Cost of the Selection Step and the Gram Expansion}

The Gram expansion of Eq.~\eqref{eq:supp-gram} lowers the cost per candidate token from $O(d\,d_h)$ to $O(d_h^2)$. Measured directly, the saving is a factor of 6.66 at $N=1024$ candidate tokens and 12.44 at $N=16384$, and working memory shrinks by a factor of 17.5, from 1121\,MiB to 64\,MiB at $N=16384$. The measured factors fall short of the theoretical $d/d_h=32$ because the matrix multiplications on the Gram path are small and run at lower hardware efficiency, whereas the explicit projection is one large multiplication and runs at high efficiency. The two paths compute the same value algebraically, but their floating-point operations run in a different order, so the scores differ minutely, at the limit of numerical precision. As a result, the top-10\% selection is identical for $N \le 8192$, while at $N = 16384$ ranks near the boundary swap and the overlap is 99.99\% (at most 1 token out of the top 1{,}638 per head in the measured sample).

The expansion accelerates only one of the three steps that make up a compression. A compression runs the peek, then the scoring, then the eviction itself, and the expansion touches only the scoring. At the longest measured context of 15{,}715 tokens the scoring shrinks from 535.6\,ms to 44.6\,ms, a factor of 12.0, which makes the whole compression, with the other two steps unchanged, 4.5 times cheaper. Both factors grow with context length, since the cost the expansion removes is proportional to the number of candidate tokens, and at 2{,}048 tokens they are only 2.3 and 1.4.

\subsubsection{Net Throughput and Memory}

Table~\ref{tab:supp-overhead} shows the throughput and KV memory after compression, measured on an RTX 3090 Ti running Llama-3.1-8B at batch 1 with the retention rate set to 10\%. Since this measurement concerns speed rather than accuracy, the head scores are synthesized rather than read from a benchmark item, and their concentration is set so that the budget is as uneven across groups as Section~\ref{supp:results:alloc} measures it to be. Turning those scores into budgets goes through Eq.~(8) unchanged. With the longest prompt measured, 15{,}715 tokens, decode becomes 8.3\% faster and attention itself 2.64 times faster. That the gain stays in the single digits is a property of the operating point rather than a limitation of the implementation. Batch-1 decode is bound by the bandwidth of reading the model weights at every step, which leaves attention at only 13\% of the step, so even eliminating KV access entirely would yield at most about a 9\% gain. The shorter the context, the smaller the share KV holds and the smaller the gain. Compression also has to earn back its own cost. One compression is a one-time expense, 139\,ms at the longest prompt, and every decode step after it saves 1.95\,ms, so the expense is recovered after 71 generated tokens. In a 2{,}048-token context the one-time expense falls to 89.4\,ms but the per-step saving falls further, to 0.294\,ms, pushing the break-even to 304 tokens. Without the Gram expansion of the previous subsection the one-time expense rises to 630\,ms and the 71-token break-even moves to about 324.

\subsubsection{Peek Step Count in Practice}

\begin{table}[t]
  \centering
  \small
  \setlength{\tabcolsep}{5pt}
  \begin{tabular}{lrccc}
    \toprule
    & & \multicolumn{2}{c}{Peek steps} & Format \\
    \cmidrule(lr){3-4}
    Benchmark & $n$ & Median & Max & skips \\
    \midrule
    RULER-16k & 1{,}300 & 2 & 4 & 1 \\
    LongBench & 3{,}750 & 1 & 13 & 0 \\
    BFCL v3 & 300 & 2 & 2 & 1 \\
    HumanEval & 164 & 5 & 6 & 4 \\
    \midrule
    All & 5{,}514 & 2 & 13 & 1 \\
    \bottomrule
  \end{tabular}
  \caption{Extra steps the peek generates before reaching the first content token. \emph{Format skips} is the median number of those steps classified as format and skipped.}
  \label{tab:supp-peeksteps}
\end{table}

Table~\ref{tab:supp-peeksteps} shows how many steps the peek actually generates before it reaches the first content token. On most benchmarks one or two steps are enough. Only HumanEval stands out with a median of 5, because code generation emits a long run of newlines and indentation right after the prompt and takes longer to reach a content token. This cost is therefore better reported per benchmark than lumped into a single range. The KV is not yet compressed during these steps, which is why peak memory does not shrink.